\def\year{2020}\relax
\documentclass[letterpaper]{article} 
\usepackage{aaai20}  
\usepackage{times}  
\usepackage{helvet} 
\usepackage{courier}  
\usepackage[hyphens]{url}  
\usepackage{graphicx} 
\usepackage{graphicx}  
\usepackage[utf8]{inputenc} 
\usepackage{url}            
\usepackage{booktabs}       
\usepackage{amsfonts}       
\usepackage{nicefrac}       
\usepackage{microtype}      

\usepackage{graphicx}
\usepackage{subcaption}
\usepackage{algorithm}
\usepackage[noend]{algpseudocode}
\usepackage{color} 
\usepackage{amsmath,amssymb,amsthm}
\usepackage{amsfonts}
\usepackage{xspace}
\usepackage{xifthen}
\usepackage{bm}
\usepackage{todonotes}
\usepackage{pgfplots}
\usepackage{multirow}
\usepackage{enumitem}

\newtheorem{mydef}{Definition}

\newtheorem{mylem}{Lemma}
\newtheorem{myprop}{Proposition}

\newcommand{\citet}[1]{\citeauthor{#1} \shortcite{#1}}

\DeclareMathOperator*{\argmax}{arg\,max}
\DeclareMathOperator*{\argmin}{arg\,min}
\newcommand{\abs}[1]{\left\lvert#1\right\rvert}

\newcommand\given[1][]{\,#1\vert\,}

\newcommand{\notd}{\ensuremath{\overline{d}}\xspace}
\newcommand{\Deltat}{\ensuremath{\widetilde{\Delta}}\xspace}

\newcommand{\rvars}[1]{\ensuremath{\mathbf{#1}}\xspace}
\newcommand{\Xs}{\rvars{X}}
\newcommand{\Ys}{\rvars{Y}}
\newcommand{\Zs}{\rvars{Z}}
\newcommand{\Vs}{\rvars{V}}
\newcommand{\Ws}{\rvars{W}}
\newcommand{\Ss}{\rvars{S}}
\newcommand{\Es}{\rvars{E}}

\newcommand{\dec}{\ensuremath{d}\xspace}
\newcommand{\decc}{\ensuremath{\bar{d}}\xspace}

\newcommand{\jstate}[1]{\ensuremath{\mathbf{#1}}\xspace}
\newcommand{\xs}{\jstate{x}}

\newcommand{\ys}{\jstate{y}}
\newcommand{\zs}{\jstate{z}}

\newcommand{\vs}{\jstate{v}}
\newcommand{\ws}{\jstate{w}}

\newcommand{\pa}[2]{%
  \ifthenelse{\isempty{#2}}%
    {\ensuremath{\theta_{#1}}\xspace}
    {\ensuremath{\theta_{#1 \given #2}}\xspace}
}

\newcommand{\data}{\ensuremath{\mathcal{D}}\xspace}

\DeclareMathOperator{\KL}{KL}
\DeclareMathOperator{\Div}{Div}

\newcommand{\guy}[1] {{\color{purple}Guy: #1 }}

\usetikzlibrary{trees}
\usetikzlibrary{arrows,shapes,shapes.geometric,backgrounds}
\usetikzlibrary{calc}
\usetikzlibrary{decorations.markings}
\usetikzlibrary{fit}

\tikzstyle{bnarrow}=[
    decoration={markings,mark=at position 1 with {\arrow[scale=1.5]{>}}},
    postaction={decorate},
    shorten >=0.7pt,
    >=latex,
    line width=0.3
]
\tikzstyle{bayesnet}=[
  >=latex, thick, auto
]
\tikzstyle{bnnode}=[
  draw,ellipse,minimum size=7mm,inner sep=1pt,font=\small
]
\tikzstyle{cpt}=[
  font=\footnotesize
]

\title{Learning Fair Naive Bayes Classifiers by\\Discovering and Eliminating Discrimination Patterns}

\author{
    YooJung Choi,\textsuperscript{\rm 1}\thanks{Equal contribution}
    Golnoosh Farnadi,\textsuperscript{\rm 2,3}\footnotemark[1]
    Behrouz Babaki,\textsuperscript{\rm 4}\footnotemark[1]
    and Guy Van den Broeck\textsuperscript{\rm 1} \\
    \textsuperscript{\rm 1}University of California, Los Angeles,
    \textsuperscript{\rm 2}Mila,
     \textsuperscript{\rm 3}Universit\'e de Montr\'eal,
    \textsuperscript{\rm 4}Polytechnique Montr\'{e}al\\
    yjchoi@cs.ucla.edu,
    farnadig@mila.quebec,
    behrouz.babaki@polymtl.ca,
    guyvdb@cs.ucla.edu
}

\begin{document}

\maketitle

\begin{abstract}
    As machine learning is increasingly used to make real-world decisions, recent research efforts aim to define and ensure fairness in algorithmic decision making. Existing methods often assume a fixed set of observable features to define individuals, but lack a discussion of certain features not being observed at test time. In this paper, we study fairness of naive Bayes classifiers, which allow partial observations. In particular, we introduce the notion of a discrimination pattern, which refers to an individual receiving different classifications depending on whether some sensitive attributes were observed. Then a model is considered fair if it has no such pattern. We propose an algorithm to discover and mine for discrimination patterns in a naive Bayes classifier, and show how to learn maximum-likelihood parameters subject to these fairness constraints. Our approach iteratively discovers and eliminates discrimination patterns until a fair model is learned. An empirical evaluation on three real-world datasets demonstrates that we can remove exponentially many discrimination patterns by only adding a small fraction of them as constraints.
\end{abstract}

\section{Introduction}

With the increasing societal impact of machine learning come increasing concerns about the fairness properties of machine learning models and how they affect decision making.
For example, concerns about fairness come up in policing~\cite{mohler2018penalized}, recidivism prediction~\cite{chouldechova2017fair}, insurance pricing~\cite{kusner2017counterfactual}, hiring~\cite{datta2015automated}, and credit rating~\cite{henderson2015credit}.
The algorithmic fairness literature has proposed various solutions, from limiting the disparate treatment of similar individuals to giving statistical guarantees on how classifiers behave towards different populations.
Key approaches include individual fairness~\cite{dwork2012fairness,zemel2013learning}, statistical parity, disparate impact and group fairness~\cite{kamishima2012fairness,feldman2015certifying,chouldechova2017fair}, counterfactual fairness \cite{kusner2017counterfactual}, preference-based fairness \cite{zafar2017parity}, relational fairness \cite{farnadi2018fairness}, and equality of opportunity \cite{hardt2016equality}.
The goal in these works is usually to assure the fair treatment of individuals or groups that are identified by sensitive attributes.

In this paper, we study fairness properties of probabilistic classifiers that represent joint distributions over the features and decision variable.
In particular, Bayesian network classifiers treat the classification or decision-making task as a probabilistic inference problem: given observed features, compute the probability of the decision variable.
Such models have a unique ability that they can naturally handle missing features, by simply marginalizing them out of the distribution when they are not observed at prediction time.
Hence, a Bayesian network classifier effectively embeds exponentially many classifiers, one for each subset of observable features.
We ask whether such classifiers exhibit patterns of discrimination where similar individuals receive markedly different outcomes purely because they disclosed a sensitive attribute.

The first key contribution of this paper is an algorithm to verify whether a Bayesian classifier is fair, or else to mine the classifier for discrimination patterns. We propose two alternative criteria for identifying the most important discrimination patterns that are present in the classifier. We specialize our pattern miner to efficiently discover discrimination patterns in naive Bayes models using branch-and-bound search. These classifiers are often used in practice because of their simplicity and tractability, and they allow for the development of effective bounds.
Our empirical evaluation shows that naive Bayes models indeed exhibit vast numbers of discrimination patterns, and that our pattern mining algorithm is able to find them by traversing only a small fraction of the search~space.

The second key contribution of this paper is a parameter learning algorithm for naive Bayes classifiers that ensures that no discrimination patterns exist in the the learned distribution. We propose a signomial programming approach to eliminate individual patterns of discrimination during maximum-likelihood learning. Moreover, to efficiently eliminate the exponential number of patterns that could exist in a naive Bayes classifier, we propose a cutting-plane approach that uses our discrimination pattern miner to find and iteratively eliminate discrimination patterns until the entire learned model is fair. Our empirical evaluation shows that this process converges in a small number of iteration, effectively removing millions of discrimination patterns. Moreover, the learned fair models are of high quality, achieving likelihoods that are close to the best likelihoods attained by models with no fairness constraints. Our method also achieves higher accuracy than other methods of learning fair naive Bayes models.

\section{Problem Formalization}\label{sec:problem}
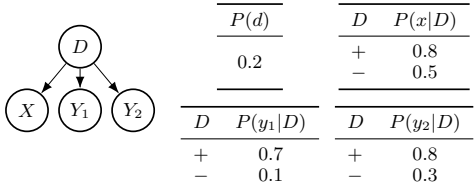
\begin{figure}[tb]
    \centering
    \scalebox{0.8}{
      \begin{tikzpicture}[bayesnet]
      
  \def\lone{15bp}
  \def\ltwo{-15bp}

  \node (D) at (-30bp,\lone) [bnnode] {$D$};
  \node (X) at (-55bp,\ltwo) [bnnode] {$X$};
  \node (Y1) at (-30bp,\ltwo) [bnnode] {$Y_1$};
  \node (Y2) at (-5bp,\ltwo) [bnnode] {$Y_2$};
  
  \begin{scope}[on background layer]
    \draw [bnarrow] (D) -- (X);
    \draw [bnarrow] (D) -- (Y1);
    \draw [bnarrow] (D) -- (Y2);
  \end{scope}
  
  \node[cpt] (d) at (50bp,15bp) {
    \footnotesize
    \begin{tabular}{c}
      \toprule
      $P(d)$\\\midrule
      \multirow{2}{*}{$0.2$}\\
      \\
      \bottomrule
    \end{tabular}
  };
  
  \node[cpt,below=0bp of d] (y1) {
    \footnotesize
    \begin{tabular}{lc}
      \toprule
      $D$ & $P(y_1|D)$\\\midrule
      $+$ & $0.7$\\
      $-$ & $0.1$\\
      \bottomrule
    \end{tabular}
  };
  
  \node[cpt,right=0bp of y1] (y2){
    \footnotesize
    \begin{tabular}{lc}
      \toprule
      $D$ & $P(y_2|D)$\\\midrule
      $+$ & $0.8$\\
      $-$ & $0.3$\\
      \bottomrule
    \end{tabular}
  };

  \node[cpt,above=0bp of y2] (x) {
    \footnotesize
    \begin{tabular}{lc}
      \toprule
      $D$ & $P(x|D)$\\\midrule
      $+$ & $0.8$\\
      $-$ & $0.5$\\
      \bottomrule
    \end{tabular}
  };  
    \end{tikzpicture}
    }
\caption{Naive Bayes classifier with a sensitive attribute $X$ and non-sensitive attributes $Y_1,Y_2$}\label{fig:nb}
\end{figure}

We use uppercase letters for random variables and lowercase letters for their assignments. Sets of variables and their joint assignments are written in bold. Negation of a binary assignment $x$ is denoted $\bar{x}$, and $\xs\!\models\! \ys$ means that $\xs$ logically implies $\ys$. Concatenation of sets $\mathbf{XY}$ denotes their~union.

Each individual is characterized by an assignment to a set of discrete variables $\Zs$, called attributes or features. Assignment $d$ to a binary decision variable $D$ represents a decision made in favor of the individual (e.g., a loan approval). A set of \emph{sensitive attributes} $\Ss \subset \Zs$ specifies a group of entities protected often by law, such as gender and race.
We now define the notion of a discrimination pattern.

\begin{mydef}
    Let $P$ be a distribution over $D \cup \Zs$. Let $\xs$ and $\ys$ be joint assignments to $\Xs\! \subseteq\! \Ss$ and $\Ys\! \subseteq \!\Zs \!\setminus\! \Xs$,  respectively.
    The \emph{degree of discrimination} of $\xs\ys$ is:
    $$\Delta_{P,d}(\xs,\ys) \triangleq P(d \given \xs\ys) - P(d \given \ys).$$
\end{mydef}
The assignment $\ys$ identifies a group of similar individuals, and the degree of discrimination quantifies how disclosing sensitive information $\xs$ affects the decision for this group.
Note that sensitive attributes missing from $\xs$ can still appear in $\ys$.
We drop the subscripts $P,d$ when clear from context. 

\begin{mydef}
    Let $P$ be a distribution over $D \cup \Zs$, and $\delta \in [0,1]$ a threshold.
    Joint assignments $\xs$ and $\ys$ form a \emph{discrimination pattern} w.r.t.\ $P$ and $\delta$ if: (1) $\Xs\!\subseteq\!\Ss$ and $\Ys\!\subseteq\!\Zs \!\setminus\! \Xs$; and (2) $\abs{\Delta_{P,d}(\xs,\ys)} > \delta$.
\end{mydef}
Intuitively, we do not want information about the sensitive attributes to significantly affect the probability of getting a favorable decision. 
Let us consider two special cases of discrimination patterns. 
First, if $\Ys\!=\!\emptyset$, then a small discrimination score $\abs{\Delta(\xs,\emptyset)}$ can be interpreted as an approximation of statistical parity, which is achieved when $P(d\given\xs) = P(d)$.
For example, the naive Bayes network in Figure~\ref{fig:nb} satisfies approximate parity for $\delta\!=\!0.2$ as $\abs{\Delta(x,\emptyset)}\!=\!0.086 \leq \delta$ and $\abs{\Delta(\bar{x},\emptyset)}\!=\!0.109 \leq \delta$.
Second, suppose $\Xs\!=\!\Ss$ and $\Ys\!=\!\Zs\!\setminus\!\Ss$. Then bounding $\abs{\Delta(\xs,\ys)}$ for all joint states $\xs$ and $\ys$ is equivalent to enforcing individual fairness assuming two individuals are considered similar if their non-sensitive attributes $\ys$ are equal.
The network in Figure~\ref{fig:nb} is also individually fair for $\delta=0.2$ because $\max_{x y_1 y_2} \abs{\Delta(x,y_1 y_2)}\!=\!0.167 \leq \delta$.\footnote{The highest discrimination score is observed at $\bar{x}$ and $y_1\bar{y_2}$, with $\Delta(\bar{x},y_1\bar{y_2}) = -0.167$.}
We discuss these connections more in Section~\ref{sec:related}.

Even though the example network is considered (approximately) fair at the group level nor at the individual level with fully observed features, it may still produce a discrimination pattern. In particular, $\abs{\Delta(\bar{x},y_1)}\!=\!0.225 > \delta$.
That is, a person with $\bar{x}$ and $y_1$ observed and the value of $Y_2$ undisclosed would receive a much more favorable decision had they not disclosed $X$ as well. Hence, naturally we wish to ensure that there exists no discrimination pattern across all subsets of observable features.

\begin{mydef}
    A distribution $P$ is \emph{$\delta$-fair} if there exists no discrimination pattern w.r.t $P$ and $\delta$.
\end{mydef}
Although our notion of fairness applies to any distribution, finding discrimination patterns can be computationally challenging: computing the degree of discrimination involves probabilistic inference, which is hard in general, and a given distribution may have exponentially many patterns.
In this paper, we demonstrate how to discover and eliminate discrimination patterns of a naive Bayes classifier effectively by exploiting its independence assumptions.
Concretely, we answer the following questions: (1) Can we certify that a classifier is $\delta$-fair?; (2) If not, can we find the most important discrimination patterns?; (3) Can we learn a naive Bayes classifier that is entirely~$\delta$-fair?


\section{Discovering Discrimination Patterns and Verifying $\delta$-fairness}
\label{sec:extractConstraints}

\begin{algorithm}[t]
\caption{\footnotesize $\textsc{Disc-Patterns}(\xs,\ys,\Es)$}\label{alg:disc-patterns}
{\footnotesize
\textbf{Input:} 
  $P$ : Distribution over $D\cup\Zs$, ~~~~$\delta$ : discrimination threshold  \hfill
\textbf{Output:} Discrimination patterns $L$ \\
\textbf{Data:} $\xs \leftarrow \emptyset$, $\ys \leftarrow \emptyset$, $\Es \leftarrow \emptyset$, $L \leftarrow []$ }
\vspace{0.05cm}
\hrule
\vspace{0.05cm}
{\footnotesize
  \begin{algorithmic}[1]
    \For {all assignments $z$ to some selected variable $Z \in \Zs \setminus \Xs\Ys\Es$~~}
        \If {$Z \in \Ss$}
            \If {$\abs{\Delta(\xs z, \ys)} > \delta$} add $(\xs z, \ys)$ to $L$ \label{line:check-delta-s} \EndIf
            \If {$\text{UB}(\xs z, \ys,\Es) > \delta$} $\textsc{Disc-Patterns}(\xs z,\ys,\Es)$ \label{line:check-ub-s} \EndIf
        \EndIf
        \If {$\abs{\Delta(\xs, \ys z)} > \delta$} add $(\xs, \ys z)$ to $L$ \EndIf \label{line:check-delta} 
        \If {$\text{UB}(\xs, \ys z,\Es) > \delta$} $\textsc{Disc-Patterns}(\xs,\ys z,\Es)$ \EndIf
    \EndFor
    \If {$\text{UB}(\xs, \ys,\Es\cup\{Z\}) > \delta$} $\textsc{Disc-Patterns}(\xs,\ys,\Es\cup\{Z\})$ \EndIf \label{line:check-ub}
  \end{algorithmic}}
\end{algorithm}

This section describes our approach to finding discrimination patterns or checking that there are none.

\subsection{Searching for Discrimination Patterns}
One may naively enumerate all possible patterns and compute their degrees of discrimination. However, this would be very inefficient as there are exponentially many subsets and assignments to consider.
We instead use branch-and-bound search to more efficiently decide if a model is fair.

Algorithm~\ref{alg:disc-patterns} finds discrimination patterns. It recursively adds variable instantiations and checks the discrimination score at each step. If the input distribution is $\delta$-fair, the algorithm returns no pattern; otherwise, it returns the set of all discriminating patterns.
Note that computing $\Delta$ requires probabilistic inference on distribution $P$. This can be done efficiently for large classes of graphical models~\cite{darwiche2009modeling,poon2011sum,dechter2013reasoning,Rahman2014Cutset,kisa2014probabilistic}, and particularly for naive Bayes networks, which will be our main focus.

Furthermore, the algorithm relies on a good upper bound to prune the search tree and avoid enumerating all possible patterns. Here, $\text{UB}(\xs,\ys,\Es)$ bounds the degree of discrimination achievable by observing more features after $\xs\ys$ while excluding features $\Es$.
\begin{myprop}\label{prop:diff-bound}
    Let $P$ be a naive Bayes distribution over $D\cup\Zs$, and let $\xs$ and $\ys$ be joint assignments to $\Xs\!\subseteq\! \Ss$ and $\Ys\! \subseteq\! \Zs\!\setminus\!\Xs$. Let $\xs_u^\prime$ (resp.\ $\xs_l^\prime$) be an assignment to $\Xs^\prime\!=\!\Ss\!\setminus\!\Xs$ that maximizes (resp.\ minimizes) $P(d\given\xs\xs^\prime)$.
    Suppose $l,u \in [0,1]$ such that $l \leq P(d \given \ys \ys^\prime) \leq u$ for all possible assignments~$\ys^\prime$ to $\Ys^\prime\!=\!\Zs \!\setminus\! (\Xs\Ys)$. 
    Then the degrees of discrimination for all patterns $\xs\xs^\prime\ys\ys^\prime$ that extend $\xs\ys$ are bounded as follows:
    \begin{align*}
        &\min_{l\leq\gamma\leq u} \Deltat\left(P(\xs\xs_l^\prime\given d),P(\xs\xs_l^\prime\given\notd),\gamma\right)
        \leq
        \Delta_{P,d}(\xs\xs^\prime,\ys\ys^\prime) \\
        &\quad\leq 
        \max_{l\leq\gamma\leq u} \Deltat\left(P(\xs\xs_u^\prime\given d),P(\xs\xs_u^\prime\given\notd),\gamma\right),
    \end{align*}
    where $\Deltat(\alpha,\beta,\gamma) \triangleq \frac{\alpha\gamma}{\alpha\gamma + \beta(1-\gamma)} - \gamma$.
\end{myprop}
Here, $\Deltat:[0,1]^3\to[0,1]$ is introduced to relax the discrete problem of minimizing or maximizing the degree of discrimination into a continuous one. In particular, $\Deltat\left(P(\xs\vert d), P(\xs\vert\notd), P(d\vert\ys)\right)$ equals the degree of discrimination $\Delta(\xs,\ys)$. This relaxation allows us to compute bounds efficiently, as closed-form solutions. We refer to the Appendix for full proofs and details.

To apply above proposition, we need to find $\xs_u^\prime, \xs_l^\prime, l, u$ by maximizing/minimizing $P(d\vert\xs\xs^\prime)$ and $P(d\vert\ys\ys^\prime)$ for a given pattern $\xs\ys$.
Fortunately, this can be done efficiently for naive Bayes classifiers.

\begin{mylem}\label{lem:nb-max}
    Given a naive Bayes distribution $P$ over $D\!\cup\!\Zs$, a subset $\Vs\!=\! \{V_i\}_{i=1}^n \!\subset\!\Zs$, and an assignment $\ws$ to $\Ws\!\subseteq\! \Zs\!\setminus\!\Vs$, we have:
    $\argmax_\vs P(d \vert \vs\ws) = \left\{\argmax_{v_i} P(v_i \vert d) / P(v_i \vert \notd) \right\}_{i=1}^n$.
\end{mylem}
That is, the joint observation $\vs$ that will maximize the probability of the decision can be found by optimizing each variable $V_i$ independently; the same holds when minimizing.
Hence, we can use Proposition~\ref{prop:diff-bound} to compute upper bounds on discrimination scores of extended patterns in linear time.

\subsection{Searching for Top-$k$ Ranked Patterns} \label{sec:most-disc-pattern}
If a distribution is significantly unfair, Algorithm~\ref{alg:disc-patterns} may return exponentially many discrimination patterns. This is not only very expensive but makes it difficult to interpret the discrimination patterns.
Instead, we would like to return a smaller set of ``interesting'' discrimination patterns. 

An obvious choice is to return a small number of discrimination patterns with the highest absolute degree of discrimination.
Searching for the $k$ most discriminating patterns can be done with a small modification to Algorithm~\ref{alg:disc-patterns}. First, the size of list $L$ is limited to $k$. The conditions in Lines~\ref{line:check-delta-s}--\ref{line:check-ub} are modified to check the current discrimination score and upper bounds against the smallest discrimination score of patterns in $L$, instead of the threshold $\delta$.

Nevertheless, ranking patterns by their discrimination score may return patterns of very low probability.
For example, the most discriminating pattern of a naive Bayes classifier learned on the COMPAS dataset\footnote{\url{https://github.com/propublica/compas-analysis}
} has a high discrimination score of 0.42, but only has a 0.02\% probability of occurring.\footnote{The corresponding pattern is $\xs\!=\!\{\text{White}, \text{Married}, \text{Female}, \allowbreak {>\!30\text{ y/o}\}}, \ys\!=\!\{\text{Probation, Pretrial}\}$.}
The probability of a discrimination pattern denotes the proportion of the population (according to the distribution) that could be affected unfairly, and thus a pattern with extremely low probability could be of lesser interest.
To address this concern, we propose a more sophisticated ranking of the discrimination patterns that also takes into account the probabilities of patterns.
\begin{mydef}
    Let $P$ be a distribution over $D \cup \Zs$. Let $\xs$ and $\ys$ be joint instantiations to subsets $\Xs \subseteq \Ss$ and $\Ys \subseteq \Zs\setminus\Xs$,  respectively.
    The \emph{divergence score} of $\xs\ys$ is:
    \begin{align}
        \Div_{P,d,\delta}(\xs,\ys) \triangleq~ \min_Q\: &\KL\left(P \;\middle\|\; Q\right) \label{eq:kld} \\
        \text{ s.t. } & \abs{\Delta_{Q,d}(\xs,\ys)} \leq \delta \nonumber \\
        & P(d\zs) = Q(d\zs), \:\forall\: d\zs \not\models \xs\ys \nonumber
    \end{align}
    where $\KL\left(P\;\middle\|\; Q\right) = \sum_{d,\zs} P(d\zs) \log(P(d\zs)/Q(d\zs))$.
\end{mydef}
The divergence score assigns to a pattern $\xs \ys$ the minimum Kullback-Leibler (KL) divergence between current distribution $P$ and a hypothetical distribution $Q$ that is fair on the pattern $\xs\ys$ and differs from $P$ only on the assignments that satisfy the pattern (namely $d\xs\ys$ and $\overline{d}\xs\ys$).
Informally, the divergence score approximates how much the current distribution $P$ needs to be changed in order for $\xs\ys$ to no longer be a discrimination pattern.
Hence, patterns with higher divergence score will tend to have not only higher discrimination score but also higher probabilities. 

For instance, the pattern with the highest divergence score
\footnote{$\xs=\{\text{Married, $>30$ y/o}\}$, $\ys=\{\}$.}
on the COMPAS dataset has a discrimination score of 0.19 which is not insignificant, but also has a relatively high probability of 3.33\% -- more than two orders of magnitude larger than that of the most discriminating pattern. Therefore, such a general pattern could be more interesting for the user studying this classifier. 


To find the top-$k$ patterns with the divergence score, we need to be able to compute the score and its upper bound efficiently. The key insights are that KLD is convex and that $Q$, in Equation~\ref{eq:kld}, can freely differ from $P$ only on one probability value (either that of $d\xs\ys$ or $\notd\xs\ys$). Then:
\begin{align}
    \Div_{P,d,\delta}(\xs,\ys) 
    = &P(d\xs\ys) \log\left( \frac{P(d\xs\ys)}{P(d\xs\ys)+r} \right) \nonumber \\
    &+ P(\notd\xs\ys) \log\left( \frac{P(\notd\xs\ys)}{P(\notd\xs\ys)-r} \right), \label{eq:kld-r}
\end{align}
where $r=0$ if $\abs{\Delta_{P,d}(\xs,\ys)}\! \leq \!\delta$; $r=\frac{\delta - \Delta_{P,d}(\xs,\ys)}{ 1/P(\xs\ys) - 1/P(\ys)}$ if $\Delta_{P,d}(\xs,\ys)\! > \!\delta$; and $r=\frac{-\delta - \Delta_{P,d}(\xs,\ys)}{ 1/P(\xs\ys) - 1/P(\ys)}$ if $\Delta_{P,d}(\xs,\ys)\! < \!-\delta$.
Intuitively, $r$ represents the minimum necessary change to $P(d\xs\ys)$ for $\xs\ys$ to be non-discriminating in the new distribution.
Note that the smallest divergence score of 0 is attained when the pattern is already fair.

Lastly, we refer to the Appendix for two upper bounds of the divergence score, which utilize the bound on discrimination score of Proposition~\ref{prop:diff-bound} and can be computed efficiently using Lemma~\ref{lem:nb-max}.

\begin{table*}[tb]
\centering
\scalebox{0.76}{
\begin{tabular}{l r r r r | r | rrr |rrr }
\toprule
    \multicolumn{5}{c|}{Dataset Statistics} & & \multicolumn{6}{c}{Proportion of search space explored} \\
       \multicolumn{5}{c|}{} & & \multicolumn{3}{c}{Divergence} & \multicolumn{3}{|c}{Discrimination} \\
   Dataset  & Size & $S$ & $N$ & \# Pat.  & $k$  &      $\delta=0.01$ &     $\delta=0.05$ &     $\delta= 0.10$ &      $\delta= 0.01$ &      $\delta= 0.05$ &      $\delta=0.10$ \\
\midrule
COMPAS & 48,834 & 4 & 3 & 15K & 1   & 6.387e-01 & 5.634e-01 & 3.874e-01 &  8.188e-03 & 8.188e-03 & 8.188e-03 \\
                                &&&&& 10  & 7.139e-01 & 5.996e-01 & 4.200e-01 &  3.464e-02 & 3.464e-02 & 3.464e-02 \\
                                &&&&& 100 & 8.222e-01 & 6.605e-01 & 4.335e-01 &  9.914e-02 & 9.914e-02 & 9.914e-02 \\
\midrule
Adult & 32,561 & 4 & 9 & 11M & 1   & 3.052e-06 & 7.260e-06 & 1.248e-05 &  2.451e-04 & 2.451e-04 & 2.451e-04 \\
                         &&&&& 10  & 7.030e-06 & 1.154e-05 & 1.809e-05 &  2.467e-04 & 2.467e-04 & 2.467e-04 \\
                         &&&&& 100 & 1.458e-05 & 1.969e-05 & 2.509e-05 &  2.600e-04 & 2.600e-04 & 2.597e-04 \\
\midrule
German & 1,000 & 4 & 16 & 23B   & 1   & 5.075e-07 & 2.731e-06 & 2.374e-06 &  7.450e-08 & 7.450e-08 & 7.450e-08 \\
                            &&&&& 10  & 9.312e-07 & 3.398e-06 & 2.753e-06 &  1.592e-06 & 1.592e-06 & 1.592e-06 \\
                            &&&&& 100 & 1.454e-06 & 4.495e-06 & 3.407e-06 &  5.897e-06 & 5.897e-06 & 5.897e-06 \\
\bottomrule
\end{tabular}
}
\caption{Data statistics (number of training instances, sensitive features $S$, non-sensitive features $N$, and potential patterns) and the proportion of patterns explored during the search, using the \emph{Divergence} and \emph{Discrimination} scores as rankings.}
\label{tab:ratio}
\end{table*}

\subsection{Empirical Evaluation of Discrimination Pattern Miner}
\label{sec:exp-miner}

In this section, we report the experimental results on the performance of our pattern mining algorithms. All experiments were run on an AMD Opteron 275 processor (2.2GHz) and 4GB of RAM running
Linux Centos 7. Execution time is limited to 1800 seconds.

\textbf{Data and pre-processing.} We use three datasets: The \emph{Adult} dataset and \emph{German} dataset are used for predicting income level and credit risk, respectively, and are obtained from the UCI machine learning repository\footnote{\url{https://archive.ics.uci.edu/ml}}; the \emph{COMPAS} dataset is used for predicting recidivism.
These datasets have been commonly studied regarding fairness and were shown to exhibit some form of discrimination by several previous works~\cite{luong2011k,larson2016we,tramer2017fairtest,salimi2019interventional}.
As pre-processing, we removed unique features (e.g. names of individuals) and duplicate features.\footnote{The processed data, code, and Appendix are available at \url{https://github.com/UCLA-StarAI/LearnFairNB}.} See Table~\ref{tab:ratio} for a summary.

\textbf{Q1. Does our pattern miner find discrimination patterns more efficiently than by enumerating all possible patterns?} 
We answer this question by inspecting the fraction of all possible patterns that our pattern miner visits during the search. Table~\ref{tab:ratio} shows the results on three datasets, using two rank heuristics (discrimination and divergence) and three threshold values (0.01, 0.05, and 0.1). The results are reported for mining the top-$k$ patterns when $k$ is 1, 10, and 100. A naive method has to enumerate all possible patterns to discover the discriminating ones, while our algorithm visits only a small fraction of patterns (e.g., one in every several millions on the German dataset). 

\begin{figure}[tb]
    \centering
    \scalebox{0.85}{
        \includegraphics[width=\linewidth]{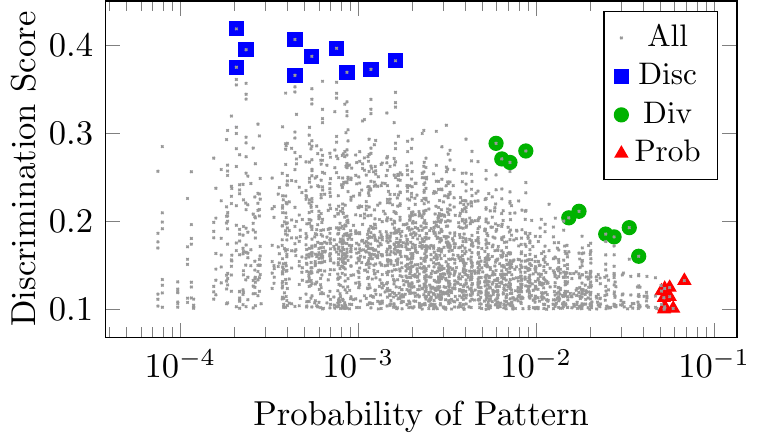}
    }
    \caption{Discrimination patterns with $\delta = 0.1$ for the max-likelihood NB classifier on COMPAS.}
    \label{fig:kld-diff}
\end{figure}

\textbf{Q2. Does the divergence score find discrimination patterns with both a high discrimination score and high probability?} 
Figure~\ref{fig:kld-diff} shows the \emph{probability} and \emph{discrimination score} of all patterns in the COMPAS dataset. The top-10 patterns according to three measures (discrimination score, divergence score, and probability) are highlighted in the figure. The observed trade-off between probability and discrimination score indicates that picking the top patterns according to each measure will yield low quality patterns according to the other measure. The divergence score, however, balances the two measures and returns patterns that have high probability and discrimination scores.
Also observe that the patterns selected by the divergence score lie in the Pareto front. This in fact always holds by the definition of this heuristic; fixing the probability and increasing the discrimination score will also increase the divergence score, and vice versa.

\section{Learning Fair Naive Bayes Classifiers}

\label{sec:learning}

We now describe our approach to learning the  maximum-likelihood parameters of a naive Bayes model from data while eliminating discrimination patterns. 
A common approach to learning naive Bayes models with certain properties is to formulate it as an optimization problem of certain form, for which efficient solvers are available~\cite{KhosraviIJCAI19}.
We formulate the learning subject to fairness constraints as a signomial program, which has the following form:
\begin{align*}
    \text{minimize}\: f_0(x), \:
    \text{~~~s.t.~~~}\: f_i(x) \leq 1,\: ~~~g_j(x) = 1 \quad \forall\: i,j
\end{align*}
where each $f_i$ is signomial while $g_j$ is monomial. 
A \emph{signomial} is a function of the form $\sum_{k} c_{k}x_1^{a_{1k}}\cdots x_n^{a_{1n}}$ defined over real positive variables $x_1\ldots x_n$ where $c_k,a_{ij} \in \mathbb{R}$; a \emph{monomial} is of the form $c x_1^{a_1} \cdots x_n^{a_n}$ where $c > 0$ and $a_i \in \mathbb{R}$. 
Signomial programs are not globally convex, but a locally optimal solution can be computed efficiently, unlike the closely related class of geometric programs, for which the globally optimum can be found efficiently \cite{ecker1980geometric}.

\subsection{Parameter Learning with Fairness Constraints}

The likelihood of a Bayesian network given data \data is $P_\theta(\data)\!=\! \prod_i \theta_i^{n_i}$ where $n_i$ is the number of examples in \data that satisfy the assignment corresponding to parameter $\theta_i$. To learn the maximum-likelihood parameters, we minimize the inverse of likelihood which is a monomial: $\theta_{\text{ml}}\! =\! \argmin_\theta \prod_i \theta_i^{-n_i}$.
The parameters of a naive Bayes network with binary class consist of $\pa{d}{},\pa{\bar{d}}{}$, and $\pa{z}{d},\pa{z}{\bar{d}}$ for all $z$.

Next, we show the constraints for our optimization problem.
To learn a valid distribution, we need to ensure that probabilities are non-negative and sum to one. The former assumption is inherent to signomial programs. 
To enforce the latter, for each instantiation~$d$ and feature $Z$, we need that $\sum_z \pa{z}{d} = 1$, or as signomial inequality constraints: $\sum_z \pa{z}{d} \leq 1$ and $2 - \sum_z \pa{z}{d} \leq 1$.

Finally, we derive the constraints to ensure that a given pattern $\xs\ys$ is non-discriminating.
\begin{myprop}\label{prop:pattern-constraint}
    Let $P_\theta$ be a naive Bayes distribution over $D\cup\Zs$, and let $\xs$ and $\ys$ be joint assignments to $\Xs \subseteq \Ss$ and $\Ys \subseteq \Zs\setminus\Xs$. Then $\abs{\Delta_{P_\theta,d}(\xs,\ys)} \leq \delta$ for a threshold $\delta \in [0,1]$ iff the following holds:
    \begin{gather*}
        r_{\xs} = \frac{\prod_x\pa{x}{\bar{d}}}{\prod_x\pa{x}{d}}, \qquad r_{\ys}=\frac{\pa{\bar{d}}{} \prod_y\pa{y}{\bar{d}}}{\pa{d}{} \prod_y\pa{y}{d}}, \\
        \left(\frac{1-\delta}{\delta}\right) r_\xs r_\ys - \left(\frac{1+\delta}{\delta}\right) r_\ys - r_\xs r_\ys^2 \leq 1, \\
        -\left(\frac{1+\delta}{\delta}\right) r_\xs r_\ys + \left(\frac{1-\delta}{\delta}\right) r_\ys - r_\xs r_\ys^2 \leq 1.
    \end{gather*}
\end{myprop}
Note that above equalities and inequalities are valid signomial program constraints. Thus, we can learn the maximum-likelihood parameters of a naive Bayes network while ensuring a certain pattern is fair by solving a signomial program.
Furthermore, we can eliminate multiple patterns by adding the constraints in Proposition~\ref{prop:pattern-constraint} for each of them. However, learning a model that is entirely fair with this approach will introduce an exponential number of constraints. Not only does this make the optimization more challenging, but listing all patterns may simply be infeasible.

\subsection{Learning $\delta$-fair Parameters}

To address the aforementioned challenge of removing an exponential number of discrimination patterns, we propose an approach based on the \emph{cutting plane} method. That is, we iterate between \emph{parameter learning} and \emph{constraint extraction}, gradually adding fairness constraints to the optimization.
The parameter learning component is as described in the previous section, where we add the constraints of Proposition~\ref{prop:pattern-constraint} for each discrimination pattern that has been extracted so far.
For constraint extraction we use the top-$k$ pattern miner presented in Section~\ref{sec:most-disc-pattern}.
At each iteration, we learn the maximum-likelihood parameters subject to fairness constraints, and find $k$ more patterns using the updated parameters to add to the set of constraints in the next iteration. This process is repeated until the pattern miner finds no more discrimination pattern.

In the worst case, our algorithm may add exponentially many fairness constraints whilst solving multiple optimization problems. However, as we will later show empirically, we can learn a $\delta$-fair model by explicitly enforcing only a small fraction of fairness constraints. The efficacy of our approach depends on strategically extracting patterns that are significant in the overall distribution. Here, we again use a ranking by discrimination or divergence score, which we also evaluate empirically.

\subsection{Empirical Evaluation of $\delta$-fair Learner}
\label{sec:exp-learner}


We will now evaluate our iterative algorithm for learning $\delta$-fair naive Bayes models. We use the same datasets and hardware as in Section~\ref{sec:exp-miner}. To solve the signomial programs, we use \emph{GPkit}, which finds local solutions to these problems using a convex optimization solver as its backend.\footnote{We use Mosek (\url{www.mosek.com}) as backend.} Throughout our experiments, Laplace smoothing was used to avoid learning zero probabilities.

\begin{figure*}[tb]
    \begin{subfigure}[t]{0.66\textwidth}
        \centering
        \includegraphics[width=\linewidth]{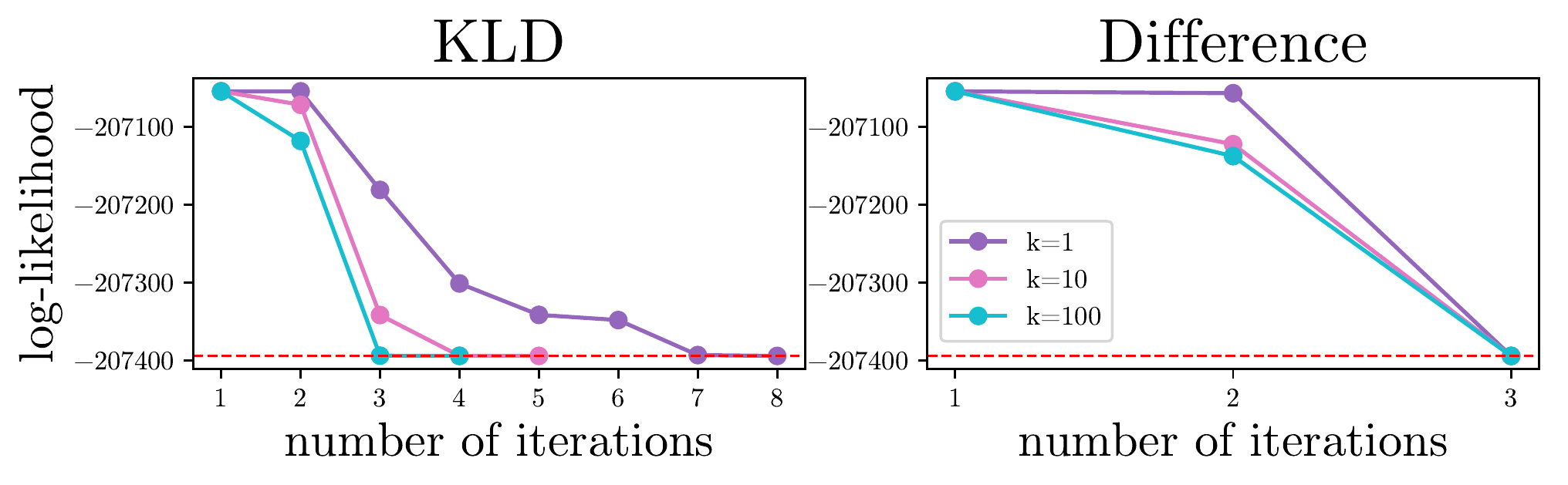}
        \caption{Log-likelihood}\label{fig:likelihood}
    \end{subfigure}
    \begin{subfigure}[t]{0.33\textwidth}
        \centering
        \includegraphics[width=\linewidth]{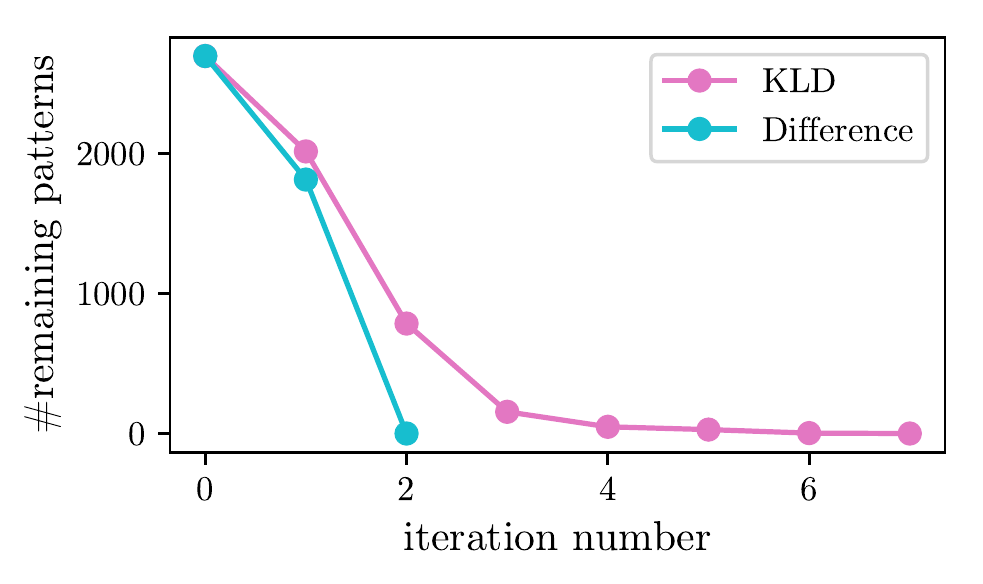}
        \caption{Number of remaining patterns}\label{fig:num_remaining}
    \end{subfigure}
    \caption{Log-likelihood and the number of remaining discrimination patterns after each iteration of learning on COMPAS dataset with $\delta=0.1$.}
\end{figure*}

\textbf{Q1. Can we learn a $\delta$-fair model in a small number of iterations while only asserting a small number of fairness constraints?}
We train a naive Bayes model on the COMPAS dataset subject to $\delta$-fairness constraints. Fig.~\ref{fig:likelihood} shows how the iterative method converges to a $\delta$-fair model, whose likelihood is indicated by the dotted line. 
Our approach converges to a fair model in a few iterations, including only a small fraction of the fairness constraints.
In particular, adding only the most discriminating pattern as a constraint at each iteration learns an entirely $\delta$-fair model with only three fairness constraints.\footnote{There are 2695 discrimination patterns w.r.t.\ unconstrained naive Bayes on COMPAS and $\delta=0.1$.}
Moreover, Fig.~\ref{fig:num_remaining} shows the number of remaining discrimination patterns after each iteration of learning with $k\!=\!1$.
Note that enforcing a single fairness constraint can eliminate a large number of remaining ones. Eventually, a few constraints subsume all discrimination patterns. 

We also evaluated our $\delta$-fair learner on the other two datasets; see Appendix for plots. We observed that more than a million discrimination patterns that exist in the unconstrained maximum-likelihood models were eliminated using a few dozen to, even in the worst case, a few thousand fairness constraints.
Furthermore, stricter fairness requirements (smaller $\delta$) tend to require more iterations, as would be expected. An interesting observation is that neither of the rankings consistently dominate the other in terms of the number of iterations to converge.

\begin{table}[tb]
    \centering
    \scalebox{0.8}{
    \begin{tabular}{l r r r }
    \toprule
        Dataset & Unconstrained & $\delta$-fair & Independent \\
        \midrule
       COMPAS &  -207,055& -207,395 & -208,639 \\
       Adult  &-226,375 & -228,763 &  -232,180 \\
       German &  -12,630 & -12,635 & -12,649\\ 
        \bottomrule
    \end{tabular}
    }
    \caption{Log-likelihood of models learned without fairness constraints, with the $\delta$-fair learner ($\delta=0.1$), and by making sensitive variables independent from the decision variable.}
    \label{tab:compare}
\end{table}

\textbf{Q2. How does the quality of naive Bayes models from our fair learner compare to ones that make the sensitive attributes independent of the decision? and to the best model without fairness constraints?}
A simple method to guarantee that a naive Bayes model is $\delta$-fair is to make all sensitive variables independent from the target value. An obvious downside is the negative effect on the predictive power of the model. We compare the models learned by our approach with:
(1) a maximum-likelihood model with no fairness constraints (unconstrained) and (2) a model in which the sensitive variables are independent of the decision variable, and the remaining parameters are learned using the max-likelihood criterion (independent).
These models lie at two opposite ends of the spectrum of the trade-off between fairness and accuracy. The $\delta$-fair model falls between these extremes, balancing approximate fairness and prediction power.

We compare the log-likelihood of these models, shown in Table~\ref{tab:compare}, as it captures the overall quality of a probabilistic classifier which can make predictions with partial observations.
The $\delta$-fair models achieve likelihoods that are much closer to those of the unconstrained models than the independent ones.
This shows that it is possible to enforce the fairness constraints without a major reduction in model quality.

\begin{table}[h]
    \centering
    \scalebox{0.8}{
    \begin{tabular}{lrrrrr}
    \toprule
     Dataset& $\lambda=$0.5 & $\lambda=$0.9 & $\lambda=$0.95 & $\lambda=$0.99 & $\lambda=$1.0 \\
    \midrule
    COMPAS & 2,504 & 2,471 & 2,470 & 3,069 & 0 \\
    Adult & \textgreater 1e6 & 661 & 652 & 605 & 0 \\
    German & \textgreater 1e6 & 3 & 2 & 0 & 0\\
    \bottomrule
    \end{tabular}
    }
    \caption{Number of remaining patterns with $\delta\!=\!0.1$ in naive Bayes models trained on discrimination-free data, where $\lambda$ determines the trade-off between fairness and accuracy in the data repair step~\cite{feldman2015certifying}.}
    \label{tab:remaining}
\end{table}

\textbf{Q3. Do discrimination patterns still occur when learning naive Bayes models from fair data?}
We first use the data repair algorithm proposed by \citet{feldman2015certifying} to remove discrimination from data, and learn a naive Bayes model from the repaired data. Table~\ref{tab:remaining} shows the number of remaining discrimination patterns in such model.
The results indicate that as long as preserving some degree of accuracy is in the objective, this method leaves lots of discrimination patterns, whereas our method removes all patterns.

\begin{table}[h]
\centering
\scalebox{0.8}{
\begin{tabular}{c c c c c}
    \toprule
    dataset & Unconstrained & 2NB & Repaired & $\delta$-fair \\
    \midrule
    COMPAS & 0.880 & 0.875 & 0.878 & 0.879\\
    Adult & 0.811 & 0.759 & 0.325 & 0.827\\
    German & 0.690  & 0.679 & 0.688 & 0.696\\
    \bottomrule
\end{tabular}
}
\caption{Comparing accuracy of our $\delta$-fair models with two-naive-Bayes method and a naive Bayes model trained on repaired, discrimination-free data.}
\label{tab:otherNB}
\end{table}
\textbf{Q4. How does the performance of $\delta$-fair naive Bayes classifier compare to existing work?}

Table~\ref{tab:otherNB} reports the 10-fold CV accuracy of our method ($\delta$-fair) compared to a max-likelihood naive Bayes model (unconstrained) and two other methods of learning fair classifiers: the two-naive-Bayes method (2NB)~\cite{calders2010three}, and a naive Bayes model trained on discrimination-free data using the repair algorithm of~\citet{feldman2015certifying} with $\lambda = 1$. 
Even though the notion of discrimination patterns was proposed for settings in which predictions are made with missing values, our method still outperforms other fair models in terms of accuracy, a measure better suited for predictions using fully-observed features.
Moreover, our method also enforces a stronger definition of fairness than the two-naive-Bayes method which aims to achieve statistical parity, which is subsumed by the notion of discrimination patterns.
It is also interesting to observe that our $\delta$-fair NB models perform even better than unconstrained NB models for the Adult and German dataset. Hence, removing discrimination patterns does not necessarily impose an extra cost on the prediction task.

\section{Related Work} \label{sec:related}

Most prominent definitions of fairness in machine learning can be largely categorized into \emph{individual fairness} and \emph{group fairness}. Individual fairness is based on the intuition that similar individuals should be treated similarly. For instance, the Lipschitz condition~\cite{dwork2012fairness} requires that the statistical distance between classifier outputs of two individuals are bounded by a task-specific distance between them.
As hinted to in Section~\ref{sec:problem}, our proposed notion of $\delta$-fairness satisfies the Lipschitz condition if two individuals who differ only in the sensitive attributes are considered similar, thus bounding the difference between their outputs by $\delta$. However, our definition cannot represent more nuanced similarity metrics that consider relationships between feature values.

Group fairness aims at achieving equality among populations differentiated by their sensitive attributes. An example of group fairness definition is statistical (demographic) parity, which states that a model is fair if the probability of getting a positive decision is equal between two groups defined by the sensitive attribute, i.e.\ $P(d \vert s)\!=\! P(d \vert \bar{s})$ where $d$ and $S$ are positive decision and sensitive variable, respectively.
Approximate measures of statistical parity include CV-discrimination score~\cite{calders2010three}: $P(d \vert s)\!-\!P(d \vert \bar{s})$; and disparate impact (or $p$\%-rule)~\cite{feldman2015certifying,zafar2017fairness}: $P(d \vert \bar{s})/P(d \vert s)$.
Our definition of $\delta$-fairness is strictly stronger than requiring a small CV-discrimination score, as a violation of (approximate) statistical parity corresponds to a discrimination pattern with only the sensitive attribute (i.e. empty $\ys$). 
Even though the $p$\%-rule was not explicitly discussed in this paper, our notion of discrimination pattern can be extended to require a small relative (instead of absolute) difference for partial feature observations (see Appendix for details). However, as a discrimination pattern conceptually represents an unfair treatment of an individual based on observing some sensitive attributes, using relative difference should be motivated by an application where the level of unfairness depends on the individual’s classification score.

Moreover, statistical parity is inadequate in detecting bias for subgroups or individuals. We resolve such issue by eliminating discrimination patterns for all subgroups that can be expressed as assignments to subsets of features.
In fact, we satisfy approximate statistical parity for any subgroup defined over the set of sensitive attributes, as any subgroup can be expressed as a union of joint assignments to the sensitive features, each of which has a bounded discrimination score.
\citet{kearns2017preventing} showed that auditing fairness at this arbitrary subgroup level (i.e. detecting \emph{fairness gerrymandering}) is computationally hard.

Other notions of group fairness include equalized true positive rates (equality of opportunity), false positive rates, or both (equalized odds~\cite{hardt2016equality}) among groups defined by the sensitive attributes. These definitions are ``oblivious'' to features other than the sensitive attribute, and focus on equalizing measures of classifier performance assuming all features are always observed. On the other hand, our method aims to ensure fairness when classifications may be made with missing features.
Moreover, our method still applies in decision making scenarios where a true label is not well defined or hard to observe.

Our approach differs from causal approaches to fairness~\cite{kilbertus2017avoiding,kusner2017counterfactual,russell2017worlds} which are more concerned with the causal mechanism of the real world that generated a potentially unfair decision, whereas we study the effect of sensitive information on a known classifier. 

There exist several approaches to learning fair naive Bayes models. First, one may modify the data to achieve fairness and use standard algorithms to learn a classifier from the modified data. For instance, \citet{kamiran2009classifying} proposed to change the labels for features near the decision boundary to achieve statistical parity, while the repair algorithm of \citet{feldman2015certifying} changes the non-sensitive attributes to reduce their correlation with the sensitive attribute. Although these methods have the flexibility of learning different models, we have shown empirically that a model learned from a fair data may still exhibit discrimination patterns.
On the other hand, \citet{calders2010three} proposed three different Bayesian network structures modified from a naive Bayes network in order to enforce statistical parity directly during learning. We have shown in the previous section that our method achieves better accuracy than their two-naive-Bayes method (which was found to be the best of three methods), while ensuring a stricter definition of fairness.
Lastly, one may add a regularizer during learning~\cite{kamishima2012fairness,zemel2013learning}, whereas we formulated to problem as constrained optimization, an approach often used to ensure fairness in other models~\cite{dwork2012fairness,kearns2017preventing}.

\section{Discussion and Conclusion}

In this paper we introduced a novel definition of fair probability distribution in terms of discrimination patterns which considers exponentially many (partial) observations of features. We have also presented algorithms to search for discrimination patterns in naive Bayes networks and to learn a high-quality fair naive Bayes classifier from data.
We empirically demonstrated the efficiency of our search algorithm and the ability to eliminate exponentially many discrimination patterns by iteratively removing a small fraction at a time.

We have shown that our approach of fair distribution implies group fairness such as statistical parity. However, ensuring group fairness in general is always with respect to a distribution and is only valid under the assumption that this distribution is truthful. While our approach guarantees some level of group fairness of naive Bayes classifiers, this is only true if the naive Bayes assumption holds. That is, the group fairness guarantees do not extend to using the classifier on an arbitrary population.

There is always a tension between three criteria of a probabilistic model: its fidelity, fairness, and tractability. Our approach aims to strike a balance between them by giving up some likelihood to be tractable (naive Bayes assumption) and more fair.
There are certainly other valid approaches: learning a more general graphical model to increase fairness and truthfulness, which would in general make it intractable, or making the model less fair in order to make it more truthful and tractable.

Lastly, real-world algorithmic fairness problems are only solved by domain experts understanding the process that generated the data, its inherent biases, and which modeling assumptions are appropriate. Our algorithm is only a tool to assist such experts in learning fair distributions: it can provide the domain expert with discrimination patterns, who can then decide which patterns need to be eliminated.

\subsubsection*{Acknowledgments}
This work is partially supported by NSF grants \#IIS-1633857, \#CCF-1837129, DARPA XAI grant \#N66001-17-2-4032, NEC Research, and gifts from Intel and Facebook Research.
Golnoosh Farnadi and Behrouz Babaki are supported by postdoctoral scholarships from IVADO through the Canada First Research Excellence Fund (CFREF) grant.

\bibliographystyle{aaai}
\bibliography{references}
\clearpage
\appendix
\allowdisplaybreaks

\section{Degree of Discrimination Bound}\label{sec:appx-diff-bound}

\subsection{Proof of Proposition~\ref{prop:diff-bound}}

We first derive how \Deltat represents the degree of discrimination $\Delta$ for some pattern $\xs\ys$. 
\begin{align*}
    &\Delta_{P,d}(\xs,\ys) 
    = P(d \given \xs\ys) - P(d \given \ys) \\
    &= \frac{P(\xs \given d) P(d \ys)}{P(\xs \given d) P(d \ys) + P(\xs \given \notd) P(\notd \ys)} - P(d \given \ys) \\
    &= \frac{P(\xs \given d) P(d \given \ys)}{P(\xs \given d) P(d \given \ys) + P(\xs \given \notd) P(\notd \given \ys)} - P(d \given \ys) \\
    &= \Deltat(P(\xs\given d), P(\xs\given\notd), P(d\given\ys))
\end{align*}

Clearly, if $l \leq \gamma \leq u$ then
$
    \min_{l \leq \gamma \leq u} \Deltat(\alpha,\beta,\gamma) \leq \Deltat(\alpha,\beta,\gamma) \leq \max_{l \leq \gamma \leq u} \Deltat(\alpha,\beta,\gamma).
$
Therefore, if $l \leq P(d \given \ys\ys^\prime) \leq u$, then the following holds for any $\xs$:
\begin{align*}
    &\min_{l \leq \gamma \leq u} \Deltat(P(\xs\given d),P(\xs\given\notd),\gamma) \\
    &\leq \Deltat(P(\xs\given d),P(\xs\given\notd),P(d \given \ys\ys^\prime))
    = \Delta_{P,d}(\xs,\ys\ys^\prime) \\
    &\leq \max_{l \leq \gamma \leq u}\Deltat(P(\xs\given d),P(\xs\given\notd),\gamma).
\end{align*}
Next, suppose $\xs_u^\prime = \argmax_{\xs^\prime} P(d \given \xs\xs^\prime)$ and $\xs_l^\prime = \argmin_{\xs^\prime} P(d \given \xs\xs^\prime)$. Then from Lemma~\ref{lem:nb-max}, we also have that $\xs_u^\prime = \argmax_{\xs^\prime} P(d \given \xs\xs^\prime \ys\ys^\prime)$ and $\xs_l^\prime = \argmin_{\xs^\prime} P(d \given \xs\xs^\prime \ys\ys^\prime)$ for any $\ys\ys^\prime$.
Therefore,
\begin{align*}
    &\min_{l\leq\gamma\leq u} \widetilde{\Delta}\left(P(\xs\xs_l^\prime\given d),P(\xs\xs_l^\prime\given\notd),\gamma\right) \\
    &\leq \Deltat\left( P(\xs\xs_l^\prime\given d),P(\xs\xs_l^\prime\given\notd),P(d \given \ys\ys^\prime) \right) \\
    &= \Delta_{P,d}(\xs\xs_l^\prime,\ys\ys^\prime)  \\
    &= P(d \given \xs\xs_l^\prime \ys\ys^\prime) - P(d \given \ys\ys^\prime) \\
    &\leq P(d \given \xs\xs^\prime \ys\ys^\prime) - P(d \given \ys\ys^\prime) \\
    &= \Delta_{P,d}(\xs\xs^\prime, \ys\ys^\prime) \\
    &\leq \Delta_{P,d}(\xs\xs_u^\prime, \ys\ys^\prime) \\
    &= \Deltat\left( P(\xs\xs_u^\prime\given d),P(\xs\xs_u^\prime\given\notd),P(d \given \ys\ys^\prime) \right)\\
    &\leq \max_{l\leq\gamma\leq u} \widetilde{\Delta}\left(P(\xs\xs_u^\prime\given d),P(\xs\xs_u^\prime\given\notd),\gamma\right).
\tag*{\qed}
\end{align*}

\subsection{Computing the Discrimination Bound}
If $\alpha=P(\xs\given d)=0$ and $\beta=P(\xs\given\notd)=0$, then the probability of $\xs$ is zero and thus $P(d \given \xs\ys)$ is ill-defined. Therefore, we will assume that either $\alpha$ or $\beta$ is nonzero.

Let us write $\Deltat_{\alpha,\beta}(\gamma) = \Deltat(\alpha,\beta,\gamma)$ to denote the function restricted to fixed $\alpha$ and $\beta$.
If $\alpha = \beta$, then $\Deltat_{\alpha,\beta} = 0$. Also, $\Deltat_{0,\beta}(\gamma) = -\gamma$ and $\Deltat_{\alpha,0}(\gamma) = 1 - \gamma$.
Thus, in the following analysis we assume $\alpha$ and $\beta$ are non-zero and distinct.

If $0 < \alpha \leq \beta \leq 1$, $\Deltat_{\alpha,\beta}$ is negative and convex in $\gamma$ within $0 \leq \gamma \leq 1$. On the other hand, if $0 < \beta \leq \alpha \leq 1$, then $\Deltat_{\alpha,\beta,\gamma}$ is positive and concave. 
This can quickly be checked using the following derivatives.
\begin{gather*}
    \frac{d}{d\gamma} \Deltat_{\alpha,\beta}(\gamma) = \frac{\alpha \beta}{\left(\alpha\gamma + \beta(1-\gamma)\right)^2} - 1, \\
    \frac{d^2}{d\gamma^2} \Deltat_{\alpha,\beta}(\gamma) = \frac{-2\alpha \beta(\alpha-\beta)}{\left(\alpha\gamma + \beta(1-\gamma)\right)^3}
\end{gather*}
Furthermore, the sign of the derivative at $\gamma = 0$ is different from that at $\gamma = 1$, and thus there must exist a unique optimum in $0 \leq \gamma \leq 1$.

Solving for $\frac{d}{d\gamma} \Deltat_{\alpha,\beta}(\gamma) = 0$, we get $\gamma = \frac{\beta \pm \sqrt{\alpha\beta}}{\beta - \alpha}$.
The solution corresponding to the feasible space $0 \leq \gamma \leq 1$ is:
$
    \gamma_{\text{opt}} = \frac{\beta - \sqrt{\alpha\beta}}{\beta - \alpha}.
$
The optimal value is derived as the following.
\begin{align*}
    &\Deltat_{\alpha,\beta}(\gamma_{\text{opt}})
    = \frac{\alpha\left(\frac{\beta - \sqrt{\alpha\beta}}{\beta-\alpha}\right)}{\left(\alpha-\beta\right) \left(\frac{\beta - \sqrt{\alpha\beta}}{\beta-\alpha}\right) + \beta} - \frac{\beta - \sqrt{\alpha\beta}}{\beta-\alpha} \\
    &= \frac{\alpha(\beta - \sqrt{\alpha\beta})}{\sqrt{\alpha\beta}(\beta-\alpha)} - \frac{\beta-\sqrt{\alpha\beta}}{\beta-\alpha} 
    = \frac{2\sqrt{\alpha\beta} - \alpha - \beta}{\beta - \alpha}
\end{align*}    

Next, suppose that the feasible space is restricted to $l \leq \gamma \leq u$. Then the optimal solution is: $\gamma_{\text{opt}}$ if $l \leq \gamma_{\text{opt}} \leq u$; $l$ if $\gamma_{\text{opt}} < l$; and $u$ if $\gamma_{\text{opt}} > u$.

\subsection{Proof of Lemma~\ref{lem:nb-max}}
Now we prove that we can maximize the posterior decision probability by maximizing each variable independently. It suffices to prove that for a single variable $V$ and all evidence $\ws$, $\argmax_v P(d \given v \ws) = \argmax_v \frac{P(v\given d)}{P(v\given\notd)}$. We first express $P(d \given v \ws)$ as the following:
\begin{align*}
    P(d \given v \ws)
    &= \frac{P(v \given d) P(d \given \ws)}{P(v \given d) P(d \given \ws) + P(v \given \notd) P(\notd \given \ws)} \\
    &= \frac{1}{1 + \frac{P(v \given \notd) P(\notd \given \ws)}{P(v \given d) P(d \given \ws)}}
\end{align*}
Then clearly,
\begin{align*}
    \argmax_v P(d \given v \ws)
    &= \argmin_v \frac{P(v \given \notd) P(\notd \given \ws)}{P(v \given d) P(d \given \ws)} \\
    &= \argmax_v \frac{P(v \given d)}{P(v \given \notd)}.
\tag*{\qed}
\end{align*}

\section{Divergence Score}

\subsection{Derivation of Equation~\ref{eq:kld-r}}\label{sec:appx-derive-kld}
We want to find the closed form solution of the optimization problem in Equation~\ref{eq:kld}.
Because $P$ and $Q$ differs only in two assignments, we can write the KL divergence as follows:
\begin{align*}
    &\KL\left(P \;\middle\|\; Q\right)
    = \sum_{d\zs} P(d\zs) \log \left(\frac{P(d\zs)}{Q(d\zs)} \right) \\
    &= P(d\xs\ys) \log\left( \frac{P(d\xs\ys)}{Q(d\xs\ys)} \right)
    + P(\notd\xs\ys) \log\left( \frac{P(\notd\xs\ys)}{Q(\notd\xs\ys)} \right)
\end{align*}

Let $r$ be the change in probability of $d\xs\ys$. That is, $r = Q(d\xs\ys) - P(d\xs\ys)$. For $Q$ to be a valid probability distribution, we must have $Q(d\xs\ys) + Q(\notd\xs\ys) = P(\xs\ys)$.
Then we have $Q(d\xs\ys) = P(d\xs\ys) + r$, and $Q(\notd\xs\ys) = P(\xs\ys) - Q(d\xs\ys) = P(\notd\xs\ys) - r$. We can then express the KL divergence between $P$ and $Q$ as a function of $P$ and $r$:
\begin{align*}
    g_{P,d,\xs,\ys}(r) \triangleq\:
    &P(d\xs\ys) \log\left( \frac{P(d\xs\ys)}{P(d\xs\ys)+r} \right) \\
    &+ P(\notd\xs\ys) \log\left( \frac{P(\notd\xs\ys)}{P(\notd\xs\ys)-r} \right)
\end{align*}

Moreover, the discrimination score of pattern $\xs\ys$ w.r.t $Q$ can be expressed using $P$ and $r$ as the following:
\begin{align*}
    &Q(d\given\xs\ys) - Q(d\given\ys)
    = \frac{P(d\xs\ys)+r}{P(\xs\ys)} - \frac{P(d\ys)+r}{P(\ys)} \\
    &= P(d\given\xs\ys) - P(d\given\ys) + r\left( \frac{1}{P(\xs\ys)} - \frac{1}{P(\ys)} \right) \\
    &= \Delta_{P,d}(\xs,\ys) + r\left( \frac{1}{P(\xs\ys)} - \frac{1}{P(\ys)} \right).
\end{align*}

The heuristic $\Div_{P,d,\delta}(\xs,\ys)$ is then written using $r$ as follows:
\begin{align}
    \min_r\: & g_{P,d,\xs,\ys}(r) \label{eq:h-as-g}\\
    \text{~s.t.~}\,\, & \abs{\Delta_{P,d}(\xs,\ys) + r\left( \frac{1}{P(\xs\ys)} - \frac{1}{P(\ys)} \right)} \leq \delta \nonumber \\
    & -P(d\xs\ys) \leq r \leq P(\notd\xs\ys) \nonumber
\end{align}

The objective function $g_{P,d,\xs,\ys}$ is convex in $r$ with its unconstrained global minimum at $r=0$. Note that this is a feasible point if and only if $\abs{\Delta_{P,d}(\xs,\ys)} \leq \delta$; in other words, when the pattern $\xs\ys$ is already fair.
Otherwise, the optimum must be either of the extreme points of the feasible space, whichever is closer to $0$.
The extreme points for the first set of inequalities are:
\begin{align*}
    r_1 = \frac{\delta - P(d\given\xs\ys) + P(d\given\ys)}{1/P(\xs\ys) - 1/P(\ys)}, \\
    r_2 = \frac{-\delta - P(d\given\xs\ys) + P(d\given\ys)}{1/P(\xs\ys) - 1/P(\ys)}.
\end{align*}
If $\Delta_{P,d}(\xs,\ys) > \delta$, then $r_2 \leq r_1 < 0$. In such case, $g(r_2) \geq g(r_1)$ and $-P(d\xs\ys) \leq r_1 \leq P(\notd\xs\ys)$ as shown below:
\begin{align*}
    r_1 &< 0 \leq P(\notd\xs\ys), \\
    -r_1 &= \frac{-\delta + P(d\given\xs\ys) - P(d\given\ys)}{1/P(\xs\ys) - 1/P(\ys)}
         \leq \frac{P(d\given\xs\ys) - P(d\given\ys)}{1/P(\xs\ys) - 1/P(\ys)} \\
         &\leq \frac{P(d\given\xs\ys) - P(d\xs\given\ys)}{1/P(\xs\ys) - 1/P(\ys)}
        = P(d\xs\ys)
\end{align*}
Similarly, if $\Delta_{P,d}(\xs,\ys) < -\delta$, then $r_1 \geq r_2 > 0$. Also, $g(r_1) \geq g(r_2)$ and $-P(d\xs\ys) \leq r_2 \leq P(\notd\xs\ys)$ as shown below:
\begin{align*}
    r_2 &> 0 \geq -P(d\xs\ys), \\
    r_2 &\leq \frac{-P(d\given\xs\ys) + P(d\given\ys)}{1/P(\xs\ys) - 1/P(\ys)} \\
        &\leq \frac{P(\notd\given\xs\ys) - P(\notd\given\ys)}{1/P(\xs\ys) - 1/P(\ys)}
        = P(\notd\xs\ys)
\end{align*}

Hence, the optimal solution $r^\star$ is
\begin{align*}
    r^\star = \begin{cases}
        0, & \text{if $\abs{\Delta_{P,d}(\xs,\ys)} \leq \delta$,} \\
        \frac{\delta - \Delta_{P,d}(\xs,\ys)}{ 1/P(\xs\ys) - 1/P(\ys)}, & \text{if $\Delta_{P,d}(\xs,\ys) > \delta$,} \\
        \frac{-\delta - \Delta_{P,d}(\xs,\ys)}{ 1/P(\xs\ys) - 1/P(\ys)}, & \text{if $\Delta_{P,d}(\xs,\ys) < -\delta$,}
    \end{cases}
\end{align*}
and the divergence score is $\Div_{P,d,\delta}(\xs,\ys) = g_{P,d,\xs,\ys}(r^\star)$.

\subsection{Upper Bounds on Divergence Score}

Here we present two upper bounds on the divergence score for pruning the search tree.
The first bound uses the observation that the hypothetical distribution $Q$ with $\Delta_{Q,d}(\xs,\ys)=0$ is always a feasible hypothetical fair distribution. 
\begin{myprop}\label{prop:kld-bound}
    Let $P$ be a Naive Bayes distribution over $D\cup\Zs$, and let $\xs$ and $\ys$ be joint assignments to $\Xs \subseteq \Ss$ and $\Ys \subseteq \Zs\setminus\Xs$.
    For all possible valid extensions $\xs^\prime$ and $\ys^\prime$, the following holds:
    \begin{align*}
        \Div_{P,d,\delta}(\xs\xs^\prime,\ys\ys^\prime)
        \leq~
        &P(d\xs\ys) \log \frac{\max_{\zs\models\xs\ys} P(d\given\zs)}{\min_{\zs\models\ys} P(d\given\zs)} \\
        &+ P(\notd\xs\ys) \log \frac{\max_{\zs\models\xs\ys} P(\notd\given\zs)}{\min_{\zs\models\ys} P(\notd\given\zs)}
    \end{align*}
\end{myprop}
\begin{proof}
    Consider the following point:
    \begin{equation*}
        r_0 = \frac{- P(d\given\xs\ys) + P(d\given\ys)}{1/P(\xs\ys) - 1/P(\ys)}.
    \end{equation*}
    First, we show that above $r_0$ is always a feasible point in Problem~\ref{eq:h-as-g}:
    \begin{align*}
        &\abs{\Delta_{P,d}(\xs,\ys) + r_0 \left( \frac{1}{P(\xs\ys)} - \frac{1}{P(\ys)} \right)} \\
        &= \abs{\Delta_{P,d}(\xs,\ys) - \Delta_{P,d}(\xs,\ys)} = 0 \leq \delta,
    \end{align*}
    \begin{align*}
        r_0 &= \frac{P(\notd\given\xs\ys) - P(\notd\given\ys)}{1/P(\xs\ys) - 1/P(\ys)} \\
            &\leq \frac{P(\notd\given\xs\ys) - P(\notd\xs\given\ys)}{1/P(\xs\ys) - 1/P(\ys)} 
            = P(\notd\xs\ys), \\
        -r_0 &= \frac{P(d\given\xs\ys) - P(d\given\ys)}{1/P(\xs\ys) - 1/P(\ys)} \\
             &\leq \frac{P(d\given\xs\ys) - P(d\xs\given\ys)}{1/P(\xs\ys) - 1/P(\ys)} 
             = P(d\xs\ys).
    \end{align*}
    
    Then the divergence score for any pattern must be smaller than $g_{P,d,\xs,\ys}(r_0)$:
    \begin{align*}
        &\Div_{P,d,\delta}(\xs,\ys)
        \leq g_{P,d,\xs,\ys}(r_0) \\
        &= P(d\xs\ys) \log \frac{P(d\given\xs\ys)}{P(d\given\overline{\xs}\ys)} + P(\notd\xs\ys) \log \frac{P(\notd\given\xs\ys)}{P(\notd\given\overline{\xs}\ys)} \\
        &\leq P(d\xs\ys) \log \frac{P(d\given\xs\ys)}{\min_\xs P(d\given\xs\ys)} + P(\notd\xs\ys) \log \frac{P(\notd\given\xs\ys)}{\min_\xs P(\notd\given\xs\ys)}.
    \end{align*}
    
    Here, we use $\overline{\xs}$ to mean that $\xs$ does not hold. In other words,
    \begin{align*}
        P(d \given \overline{\xs}\ys) &= \frac{P(d\ys) - P(d\xs\ys)}{P(\ys)-P(\xs\ys)}
        = \sum_{\xs} P(d \given \xs \ys) P(\xs \given \overline{\xs}\ys).
    \end{align*}
    
    We can then use this to bound the divergence score any pattern extended from $\xs\ys$:
    \begin{align*}
        &\Div_{P,d,\delta}(\xs\xs^\prime,\ys\ys^\prime) \\
        &\leq P(d\xs\xs^\prime\ys\ys^\prime) \log \frac{P(d\given\xs\xs^\prime\ys\ys^\prime)}{\min_{\xs\xs^\prime} P(d\given \xs\xs^\prime\ys\ys^\prime)} \\
        &\quad + P(\notd\xs\xs^\prime\ys\ys^\prime) \log \frac{P(\notd\given\xs\xs^\prime\ys\ys^\prime)}{\min_{\xs\xs^\prime} P(\notd\given\xs\xs^\prime\ys\ys^\prime)} \\
        &\leq P(d\xs\ys) \log \frac{\max_{\zs\models\xs\ys} P(d\given\zs)}{\min_{\zs\models\ys} P(d\given\zs)} \\
            &\quad+ P(\notd\xs\ys) \log \frac{\max_{\zs\models\xs\ys} P(\notd\given\zs)}{\min_{\zs\models\ys} P(\notd\given\zs)}.
    \end{align*}
\end{proof}

We can also bound the divergence score using the maximum and minimum possible discrimination scores shown in Proposition~\ref{prop:diff-bound}, in place of the current pattern's discrimination. Let us denote the bounds for discrimination score as follows:
\begin{align*}
    \overline{\Delta}(\xs,\ys) &= \max_{l\leq\gamma\leq u} \Deltat\left(P(\xs\xs_u^\prime\given d),P(\xs\xs_u^\prime\given\notd),\gamma\right), \\
    \underline{\Delta}(\xs,\ys) &= \min_{l\leq\gamma\leq u} \Deltat\left(P(\xs\xs_l^\prime\given d),P(\xs\xs_l^\prime\given\notd),\gamma\right).
\end{align*}
\begin{myprop}\label{prop:kld-diff-bound}
    Let $P$ be a Naive Bayes distribution over $D\cup\Zs$, and let $\xs$ and $\ys$ be joint assignments to $\Xs \subseteq \Ss$ and $\Ys \subseteq \Zs\setminus\Xs$.
    For all possible valid extensions $\xs^\prime$ and $\ys^\prime$, $\Div_{P,d,\delta}(\xs\xs^\prime,\ys\ys^\prime) \leq \max\left( g_{P,d,\xs\xs^\prime,\ys\ys^\prime}(r_u), g_{P,d,\xs\xs^\prime,\ys\ys^\prime}(r_l) \right)$ where
    \begin{align*}
        r_u = \frac{\delta - \overline{\Delta}(\xs,\ys)}{ 1/P(\xs\xs^\prime\ys\ys^\prime) - 1/P(\ys\ys^\prime)}, \\
        r_l = \frac{-\delta - \underline{\Delta}(\xs,\ys)}{ 1/P(\xs\xs^\prime\ys\ys^\prime) - 1/P(\ys\ys^\prime)}.
    \end{align*}
\end{myprop}
\begin{proof}
    The proof proceeds by case analysis on the discrimination score of extended patterns $\xs\xs^\prime\ys\ys^\prime$.
    
    First, if $\abs{\Delta(\xs\xs^\prime,\ys\ys^\prime)} \leq \delta$, $\Div_{P,d,\delta}(\xs\xs^\prime,\ys\ys^\prime) = 0$ which is the global minimum, and thus is smaller than both $g(r_u)$ and $g(r_l)$.
    
    Next, suppose $\Delta(\xs\xs^\prime,\ys\ys^\prime) > \delta$. Then from Proposition~\ref{prop:diff-bound},
    \begin{align*}
        &r_u = \frac{\delta - \overline{\Delta}(\xs,\ys)}{ 1/P(\xs\xs^\prime\ys\ys^\prime) - 1/P(\ys\ys^\prime)} \\
        &\leq
        r^\star = \frac{\delta - \Delta_{P,d}(\xs\xs^\prime,\ys\ys^\prime)}{ 1/P(\xs\xs^\prime\ys\ys^\prime) - 1/P(\ys\ys^\prime)}
        < 0.
    \end{align*}
    As $g$ is convex with its minimum at 0, we can conclude $\Div_{P,d,\delta}(\xs\xs^\prime,\ys\ys^\prime) = g(r^\star) \leq g(r_u)$.
    
    Finally, if $\Delta(\xs\xs^\prime,\ys\ys^\prime) < -\delta$, we have
    \begin{align*}
        &r_l = \frac{-\delta - \underline{\Delta}(\xs,\ys)}{ 1/P(\xs\xs^\prime\ys\ys^\prime) - 1/P(\ys\ys^\prime)} \\
        &\geq
        r^\star = \frac{-\delta - \Delta_{P,d}(\xs\xs^\prime,\ys\ys^\prime)}{ 1/P(\xs\xs^\prime\ys\ys^\prime) - 1/P(\ys\ys^\prime)}
        > 0.
    \end{align*}
    Similarly, this implies $\Div_{P,d,\delta}(\xs\xs^\prime,\ys\ys^\prime) = g(r^\star) \leq g(r_l)$. Because the divergence score is always smaller than either $g(r_u)$ or $g(r_l)$, it must be smaller than $\max(g(r_u), g(r_l))$.
\end{proof}

Lastly, we show how to efficiently compute an upper bound on $g_{P,d,\xs\xs^\prime,\ys\ys^\prime}(r_u)$ $g_{P,d,\xs\xs^\prime,\ys\ys^\prime}(r_l)$ from Proposition~\ref{prop:kld-diff-bound} for all patterns extended from $\xs\ys$. This is necessary for pruning during the search for discrimination patterns with high divergence scores.
First, note that $r_u$ and $r_l$ can be expressed as
\begin{equation}
    \frac{c}{ 1/P(\xs\xs^\prime\ys\ys^\prime) - 1/P(\ys\ys^\prime)}, 
\end{equation}
where $c = \delta - \overline{\Delta}(\xs,\ys)$ for $r_u$ and $c = -\delta - \underline{\Delta}(\xs,\ys)$ for $r_l$. Hence, it suffices to derive the following bound.
\begin{align*}
    &g_{P,d,\xs\xs^\prime,\ys\ys^\prime}\left(\frac{c}{ 1/P(\xs\xs^\prime\ys\ys^\prime) - 1/P(\ys\ys^\prime)}\right) \\
    &= P(d\xs\xs^\prime\ys\ys^\prime) \log\left(\frac{P(d\xs\xs^\prime\ys\ys^\prime)}{P(d\xs\xs^\prime\ys\ys^\prime) + \frac{c}{ 1/P(\xs\xs^\prime\ys\ys^\prime) - 1/P(\ys\ys^\prime)}}\right) \\
        &\quad+ P(\notd\xs\xs^\prime\ys\ys^\prime) \log\left(\frac{P(\notd\xs\xs^\prime\ys\ys^\prime)}{P(\notd\xs\xs^\prime\ys\ys^\prime) - \frac{c}{ 1/P(\xs\xs^\prime\ys\ys^\prime) - 1/P(\ys\ys^\prime)}}\right) \\
    &= P(d\xs\xs^\prime\ys\ys^\prime) \log\left(\frac{P(d\given\xs\xs^\prime\ys\ys^\prime)(1-P(\xs\xs^\prime\given\ys\ys^\prime))}{P(d\given\xs\xs^\prime\ys\ys^\prime)(1-P(\xs\xs^\prime\given\ys\ys^\prime)) + c}\right) \\
        &\quad+ P(\notd\xs\xs^\prime\ys\ys^\prime) \log\left(\frac{P(\notd\given\xs\xs^\prime\ys\ys^\prime)(1-P(\xs\xs^\prime\given\ys\ys^\prime))}{P(\notd\given\xs\xs^\prime\ys\ys^\prime)(1-P(\xs\xs^\prime\given\ys\ys^\prime)) - c}\right) \\
    &\leq \begin{cases}
        0 &\text {if $c = 0$} \\
        P(d\xs\ys)\log \frac{(\max_{\zs\models\xs\ys}P(d\given\zs))(1-\min_{\xs^\prime\ys^\prime}P(\xs\xs^\prime\given\ys\ys^\prime))}{(\min_{\zs\models\xs\ys}P(d\given\zs))(1-\max_{\xs^\prime\ys^\prime}P(\xs\xs^\prime\given\ys\ys^\prime))+c} &\text{if $c < 0$} \\
        P(\notd\xs\ys)\log \frac{(\max_{\zs\models\xs\ys}P(\notd\given\zs))(1-\min_{\xs^\prime\ys^\prime}P(\xs\xs^\prime\given\ys\ys^\prime))}{(\min_{\zs\models\xs\ys}P(\notd\given\zs))(1-\max_{\xs^\prime\ys^\prime}P(\xs\xs^\prime\given\ys\ys^\prime))-c} &\text{if $c > 0$}
    \end{cases}
\end{align*}

\section{Proof of Proposition~\ref{prop:pattern-constraint}}
The probability values of positive decision in terms of naive Bayes parameters $\theta$ are as follows:
\begin{align*}
    P_\theta(\dec \given \xs\ys)
    &= \frac{P_\theta(\dec\xs\ys)}{P_\theta(\xs\ys)} \\
    &= \frac{\pa{d}{} \prod_x\pa{x}{d} \prod_y\pa{y}{d}}{\pa{d}{} \prod_x\pa{x}{d} \prod_y\pa{y}{d} + \pa{\bar{d}}{} \prod_x\pa{x}{\bar{d}} \prod_y\pa{y}{\bar{d}}} \\
    &= \frac{1}{1+\frac{\pa{\bar{d}}{} \prod_x\pa{x}{\bar{d}} \prod_y\pa{y}{\bar{d}}}{\pa{d}{} \prod_x\pa{x}{d} \prod_y\pa{y}{d}}}, \\
    P_\theta(\decc \given \ys)
    &= \frac{P_\theta(\dec\ys)}{P_\theta(\ys)}
    = \frac{1}{1+\frac{\pa{\bar{d}}{} \prod_y\pa{y}{\bar{d}}}{\pa{d}{} \prod_y\pa{y}{d}}}.
\end{align*}
For simplicity of notation, let us write:
\begin{equation}
    r_{\xs} = \frac{\prod_x\pa{x}{\bar{d}}}{\prod_x\pa{x}{d}}, \quad r_{\ys}=\frac{\pa{\bar{d}}{} \prod_y\pa{y}{\bar{d}}}{\pa{d}{} \prod_y\pa{y}{d}}. \label{eq:aux-r}
\end{equation}
Then the degree of discrimination is $\Delta_{P_\theta,d}(\xs,\ys) = P_\theta(\dec \given \xs\ys) - P_\theta(\dec \given \ys) = \frac{1}{1+r_\xs r_\ys} - \frac{1}{1+r_\ys}$. 
Now we express the fairness constraint $\abs{\Delta_{P_\theta,d}(\xs,\ys)} \leq \delta$ as the following two inequalities:
\begin{align*}
    -\delta \leq \frac{(1+r_\ys)-(1+r_\xs r_\ys)}{(1+r_\xs r_\ys)\cdot(1+r_\ys)} \leq \delta.
\end{align*}
After simplifying,
\begin{align*}
    r_\ys - r_\xs r_\ys & \geq -\delta ( 1 + r_\xs r_\ys + r_\ys + r_\xs r_\ys^2 ), \\
    r_\ys - r_\xs r_\ys & \leq \delta ( 1 + r_\xs r_\ys + r_\ys + r_\xs r_\ys^2 ).
\end{align*}
We further express this as the following two signomial inequality constraints:
\begin{gather*}
    \left(\frac{1-\delta}{\delta}\right) r_\xs r_\ys - \left(\frac{1+\delta}{\delta}\right) r_\ys - r_\xs r_\ys^2 \leq 1, \\
    -\left(\frac{1+\delta}{\delta}\right) r_\xs r_\ys + \left(\frac{1-\delta}{\delta}\right) r_\ys - r_\xs r_\ys^2 \leq 1 \label{eq:fairconstraint}
\end{gather*}
Note that $r_\xs$ and $r_\ys$ according to Equation~\ref{eq:aux-r} are monomials of $\theta$, and thus above constraints are also signomial with respect to the optimization variables $\theta$. \qed

\section{Additional Experiments}\label{sec:appx-experiments}
\begin{figure*}[tb]
    \centering
    \includegraphics[width=\linewidth]{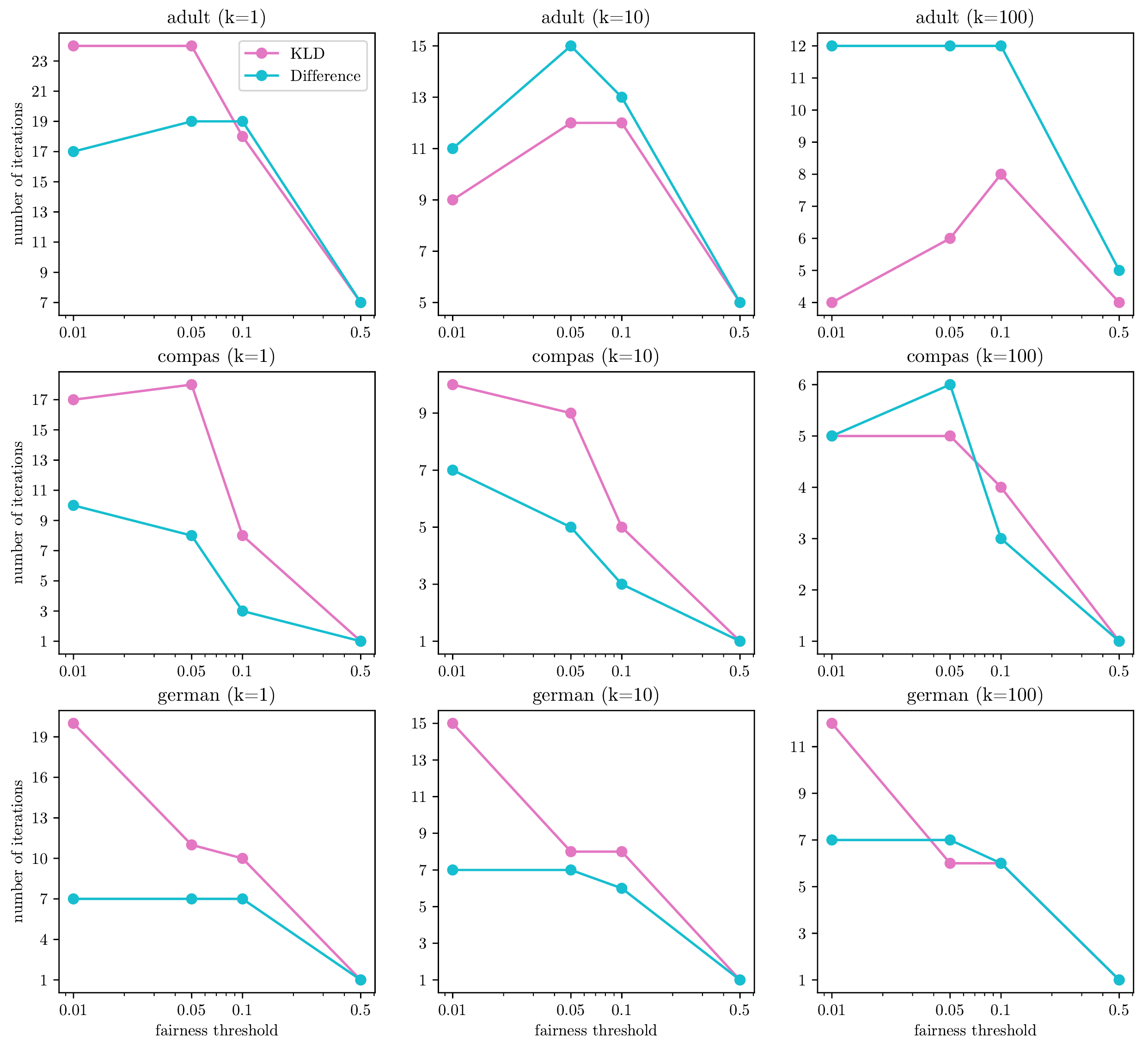}
    \caption{Number of iterations of $\delta$-fair learner until convergence}
    \label{fig:iterations}
\end{figure*}
Here we present the full set of experiments referred to in Q1 of Section~\ref{sec:exp-learner}.
\end{document}